%% file: neurips_2019.tex
\documentclass{article}

\PassOptionsToPackage{numbers, compress}{natbib}

\usepackage[preprint]{neurips_2019}

\usepackage[utf8]{inputenc} 
\usepackage[T1]{fontenc}    
\usepackage{hyperref}       
\usepackage{url}            
\usepackage{booktabs}       
\usepackage{amsfonts}       
\usepackage{nicefrac}       
\usepackage{microtype}      

\usepackage{bbding}

\usepackage{multirow}
\usepackage{graphicx}
\usepackage{comment}

\usepackage{algorithm,algpseudocode}

\usepackage{caption}

\usepackage{amsmath,amssymb,amsthm}
\newtheorem{theorem}{Theorem}
\newtheorem{lemma}[theorem]{Lemma}

\newtheorem{corollary}[theorem]{Corollary}
\newtheorem{assumption}[theorem]{Assumption}

\newcommand{\A}{\tilde{A}_\text{rw}}
\renewcommand{\L}{\tilde{L}_\text{rw}}
\newcommand{\E}{\mathbb{E}}
\renewcommand{\vec}{\mathbf}

\usepackage{color}
\usepackage{xcolor}

\title{Revisiting Graph Neural Networks:\\All We Have is Low-Pass Filters}

\author{%
  Hoang NT\thanks{gihub.com/gear/gfnn}\\
  Tokyo Institute of Technology, RIKEN\\
  \texttt{hoang.nguyen.rh@riken.jp} \\
  \And
  Takanori Maehara \\
  RIKEN \\
  \texttt{takanori.maehara@riken.jp} \\
}

\begin{document}

\definecolor{ror}{HTML}{e76f51}
\definecolor{bor}{HTML}{f4a261}
\definecolor{yel}{HTML}{e9c46a}
\definecolor{gre}{HTML}{2a9d8f}
\definecolor{blu}{HTML}{264653} 

\maketitle

\begin{abstract}
 
  Graph neural networks have become one of the most important techniques to solve machine learning problems on graph-structured data. Recent work on vertex classification proposed deep and distributed learning models to achieve high performance and scalability. However, we find that the feature vectors of benchmark datasets are already quite informative for the classification task, and the graph structure only provides a means to denoise the data. In this paper, we develop a theoretical framework based on graph signal processing for analyzing graph neural networks. Our results indicate that graph neural networks only perform low-pass filtering on feature vectors and do not have the non-linear manifold learning property. We further investigate their resilience to feature noise and propose some insights on GCN-based graph neural network design.
  

\end{abstract}

\section{Introduction}

Graph neural networks (GNN) belong to a class of neural networks which can learn from graph-structured data. Recently, graph neural networks for vertex classification and graph isomorphism test have achieved excellent results on several benchmark datasets and continuously set new state-of-the-art performance~\cite{abu2019mixhop,gcn,gat,gin}. Started with the early success of ChebNet~\cite{chebynets} and GCN~\cite{gcn} at vertex classification, many variants of GNN have been proposed to solve problems in social networks~\cite{graphsage, gaan}, biology~\cite{gat, dgi}, chemistry~\cite{fout2017protein,gilmer2017neural}, natural language processing~\cite{bastings2017graph,zhang2018graph}, computer vision~\cite{santoro2017simple}, and weakly-supervised learning~\cite{garcia2017few}.

In semi-supervised vertex classification, we observe that the parameters of a graph convolutional layer in a Graph Convolutional Network (GCN)~\cite{gcn} only contribute to overfitting. Similar observations have been reported in both simple architectures such as SGC~\cite{wu2019simplifying} and more complex ones such as DGI~\cite{dgi}. Based on this phenomenon, \citet{wu2019simplifying} proposed to view graph neural networks simply as feature propagation and propose an extremely efficient model with state-of-the-art performance on many benchmark datasets. \citet{kawamoto2018mean} made a related theoretical remark on untrained GCN-like GNNs under graph partitioning settings. From these previous studies, a question naturally arises: \emph{Why and when do graph neural networks work well for vertex classification?} In other words, is there a condition on the vertex feature vectors for graph neural network models to work well even without training? Consequently, can we find realistic counterexamples in which baseline graph neural networks (e.g. SGC or GCN) fail to perform? 

In this study, we provide an answer to the aforementioned questions from the graph signal processing perspective~\cite{ortega2018graph}. Formally, we consider a semi-supervised learning problem on a graph. Given a graph $\mathcal{G} = (\mathcal{V}, \mathcal{E})$, each vertex $i \in \mathcal{V}$ has a feature $\vec{x}(i) \in \mathcal{X}$ and label $y(i) \in \mathcal{Y}$, where $\mathcal{X}$ is a $d$-dimensional Euclidean space $\mathbb{R}^d$ and $\mathcal{Y} = \mathbb{R}$ for regression and $\mathcal{Y} = \{1, \dots, c\}$ for classification. The task is to learn a hypothesis to predict the label $y(i)$ from the feature $\vec{x}(i)$. We then characterize the graph neural networks solution to this problem and provide insights to the mechanism underlying the most commonly used baseline model GCN~\cite{gcn}, and its simplified variant SGC~\cite{wu2019simplifying}.

\begin{figure}
\centering
\begin{minipage}{.5\textwidth}
  \centering
  \includegraphics[width=0.725\linewidth]{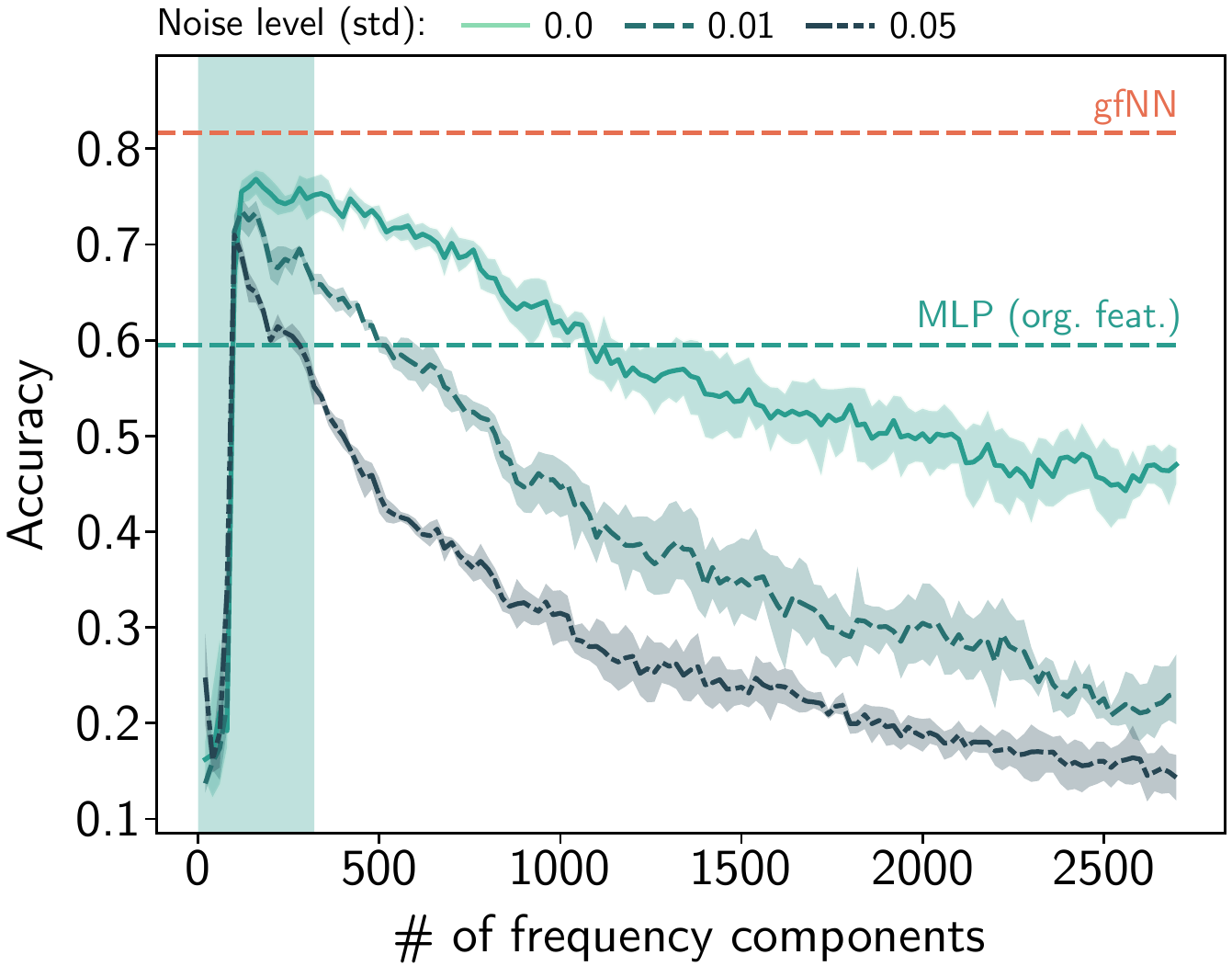}
  \captionof{figure}{Accuracy by frequency components}
  \label{f:low-freq-cora}
\end{minipage}%
\begin{minipage}{.5\textwidth}
  \centering
  \includegraphics[width=0.7\linewidth]{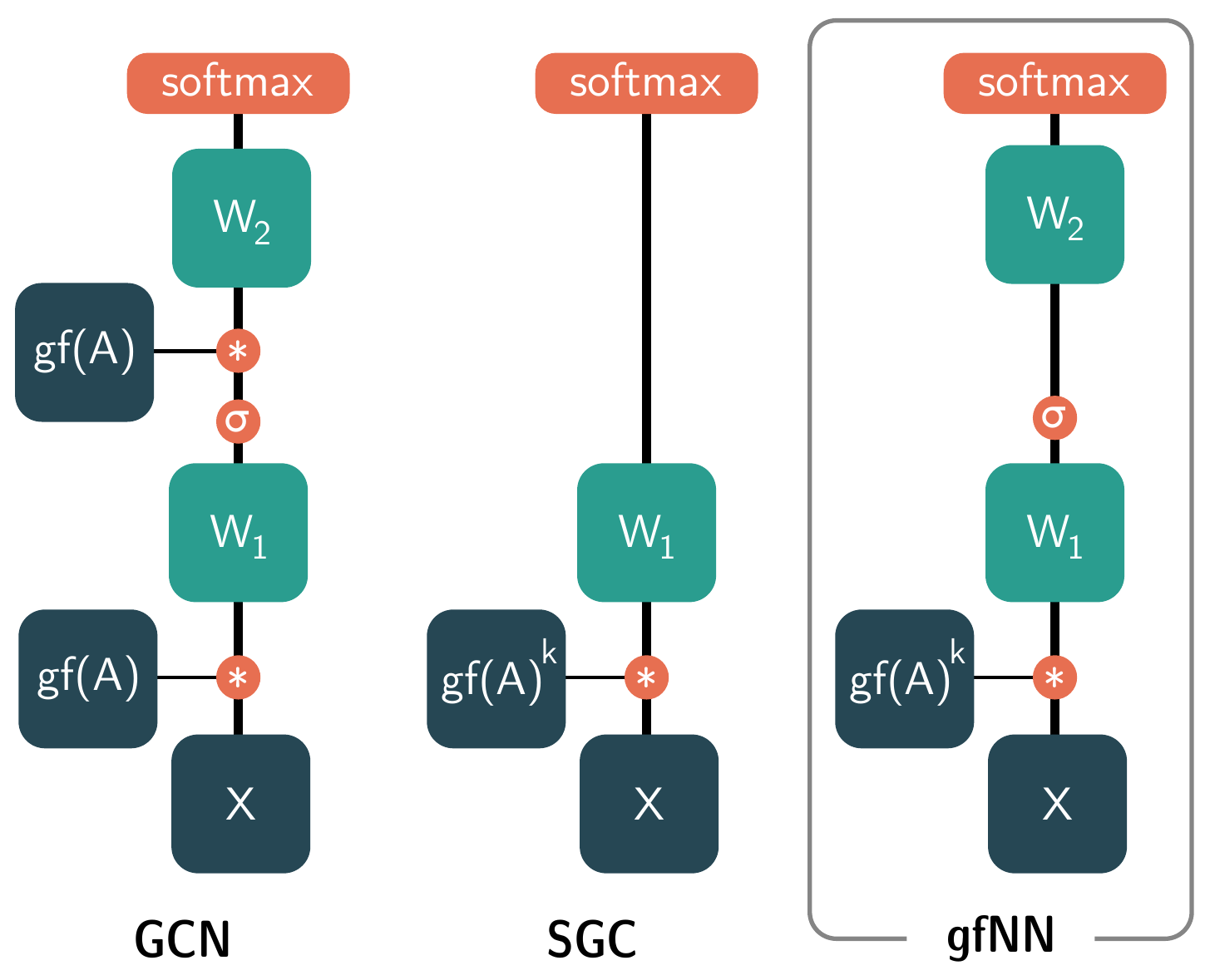}
  \captionof{figure}{A simple realization of gfNN}
  \label{f:gfnn}
\end{minipage}
\vspace{-1.5em}
\end{figure}

Graph signal processing (GSP) regards data on the vertices as signals and applies signal processing techniques to understand the signal characteristics. By combining signals (feature vectors) and graph structure (adjacency matrix or its transformations), GSP has inspired the development of learning algorithms on graph-structured data~\cite{shuman2012emerging}. In a standard signal processing problem, it is common to assume the observations contain some noise and the underlying ``true signal'' has low-frequency~\cite{rabiner1975theory}. Here, we pose a similar assumption for our problem.

\begin{assumption}
\label{asmp:low-frequency}
Input features consist of low-frequency true features and noise. The true features have sufficient information for the machine learning task.
\end{assumption}

\textbf{\textcolor{gre}{Our first contribution}} is to verify Assumption~\ref{asmp:low-frequency} on commonly used datasets (Section~\ref{sec:low-freq}). Figure~\ref{f:low-freq-cora} shows the performance of 2-layers perceptrons (MLPs) trained on features with different numbers of frequency components. In all benchmark datasets, we see that only a small number of frequency components contribute to learning. Adding more frequency components to the feature vectors only decreases the performance. In turn, the classification accuracy becomes even worse when we add Gaussian noise $\mathcal{N}(0, \sigma^2)$ to the features.

Many recent GNNs were built upon results from graph signal processing. The most common practice is to multiply the \emph{(augmented) normalized adjacency} matrix $I - \tilde{\mathcal{L}}$ with the feature matrix~$\mathcal{X}$. The product $(I - \tilde{\mathcal{L}})\mathcal{X}$ is understood as features averaging and propagation~\cite{graphsage,gcn,wu2019simplifying}. In graph signal processing literature, such operation filters signals on the graph without explicitly performing eigendecomposition on the normalized Laplacian matrix, which requires $O(n^3)$ time~\cite{vaseghi2008advanced}. Here, we refer to this augmented normalized adjacency matrix and its variants as graph filters and propagation matrices interchangeably.

\textbf{\textcolor{gre}{Our second contribution}} shows that multiplying graph signals with propagation matrices corresponds to low-pass filtering (Section~\ref{sec:lowpass}, esp., Theorem~\ref{thm:spectrum-shrinking}). Furthermore, we also show that the matrix product between the observed signal and the low-pass filters is the analytical solution to the true signal optimization problem. In contrast to the recent design principle of graph neural networks~\cite{abu2019mixhop,gcn,gin}, our results suggest that the graph convolution layer is simply low-pass filtering. Therefore, learning the parameters of a graph convolution layer is unnecessary. \citet[Theorem~1]{wu2019simplifying} also addressed a similar claim by analyzing the spectrum-truncating effect of the augmented normalized adjacency matrix. We extend this result to show that all eigenvalues monotonically shrink, which further explains the implications of the spectrum-truncating effect.

Based on our theoretical understanding, we propose a new baseline framework named \emph{gfNN (graph filter neural network)} to empirically analyze the vertex classification problem. gfNN consists of two steps: 1. Filter features by multiplication with graph filter matrices; 2. Learn the vertex labels by a machine learning model. We demonstrate the effectiveness of our framework using a simple realization model in Figure~\ref{f:gfnn}.

\textbf{\textcolor{gre}{Our third contribution}} is the following Theorem:
\begin{theorem}[Informal, see Theorem~\ref{thm:MLP-vs-gfNN}, \ref{thm:MLP-vs-GCN}]
\label{thm:main}
Under Assumption~\ref{asmp:low-frequency}, the outcomes of SGC, GCN, and gfNN are similar to those of the corresponding NNs using true features.
\end{theorem}
Theorem~\ref{thm:MLP-vs-gfNN} implies that, under Assumption~\ref{asmp:low-frequency}, both gfNN and GCN~\cite{gcn} have similar high performance. Since gfNN does not require multiplications of the adjacency matrix during the learning phase, it is much faster than GCN. In addition, gfNN is also more noise tolerant.

Finally, we compare our gfNN to the SGC model~\cite{wu2019simplifying}. While SGC is also computationally fast and accurate on benchmark datasets, our analysis implies it would fail when the feature input is nonlinearly-separable because the graph convolution part does not contribute to non-linear manifold learning. We created an artificial dataset to empirically demonstrate this claim. 

\section{Graph Signal Processing}

In this section, we introduce the basic concepts of graph signal processing.
We adopt a recent formulation~\cite{girault2018irregularity} of graph Fourier transform on irregular graphs.

Let $\mathcal{G} = (\mathcal{V}, \mathcal{E})$ be a simple undirected graph, where $\mathcal{V} = \{ 1, \dots, n \}$ be the set of $n \in \mathbb{Z}$ vertices and $\mathcal{E}$ be the set of edges.%
\footnote{We only consider unweighted edges but it is easily adopted to positively weighted edges.}
Let $A = (a_{ij}) \in \mathbb{R}^{n \times n}$ be the adjacency matrix of $G$,
$D = \mathrm{diag}(d(1), \dots, d(n)) \in \mathbb{R}^{n \times n}$ be the degree matrix of $G$, where $d(i) = \sum_{j \in \mathcal{V}} a(i,j)$ is the degree of vertex $i \in \mathcal{V}$.
$L = D - A \in \mathbb{R}^{n \times n}$ be the combinatorial Laplacian of $G$, 
$\mathcal{L} = I - D^{-1/2} A D^{-1/2}$ be the normalized Laplacian of $G$, where $I \in \mathbb{R}^{n \times n}$ is the identity matrix,
and $L_\text{rw} = I - D^{-1} A$ be the random walk Laplacian of $G$.
Also, for $\gamma \in \mathbb{R}$ with $\gamma > 0$, let $\tilde A = A + \gamma I$ be the \emph{augmented adjacency matrix}, which is obtained by adding $\gamma$ self loops to $G$, $\tilde D = D + \gamma I$ be the corresponding augmented degree matrix, and $\tilde L = \tilde D - \tilde A = L$, $\tilde{\mathcal{L}} = I - \tilde D^{-1/2} \tilde A \tilde D^{-1/2}$, $\L = I - \tilde D^{-1} \tilde A$ be the corresponding augmented combinatorial, normalized, and random walk Laplacian matrices.


A vector $x \in \mathbb{R}^n$ defined on the vertices of the graph is called a \emph{graph signal}. 
To introduce a graph Fourier transform, we need to define two operations, \emph{variation} and \emph{inner product}, on the space of graph signals.
Here, we define the variation $\Delta \colon \mathbb{R}^n \to \mathbb{R}$ and the $\tilde D$-inner product by
\begin{align}
    \Delta(x) := \sum_{(i, j) \in \mathcal{E}} (x(i) - x(j))^2 = x^\top L x; \ \
    (x, y)_{\tilde D} := \sum_{i \in \mathcal{V}} (d(i) + \gamma) x(i) y(i) = x^\top \tilde D y.
\end{align}
We denote by $\| x \|_{\tilde D} := \sqrt{ (x, x)_{\tilde D} }$ the norm induced by $\tilde D$.
Intuitively, the variation $\Delta$ and the inner product $(\cdot, \cdot)_{\tilde D}$ specify how to measure the smoothness and importance of the signal, respectively. 
In particular, our inner product puts more importance on high-degree vertices, where larger $\gamma$ closes the importance more uniformly.
We then consider the \emph{generalized eigenvalue problem (variational form)}:
Find $u_1, \dots, u_n \in \mathbb{R}^n$ such that for each $i \in \{1, \dots, n\}$, $u_i$ is a solution to the following optimization problem:
\begin{align}
\label{eq:generalized-eigenvalue-problem-variational-form}
    \text{minimize} \ \Delta(u) \ \ 
    \text{subject to} \ (u, u)_{\tilde D} = 1, \ (u, u_j)_{\tilde D} = 0, \ j \in \{1, \dots, n \}.
\end{align}
The solution $u_i$ is called an $i$-th generalized eigenvector and the corresponding objective value $\lambda_i := \Delta(u_i)$ is called the $i$-th \emph{generalized eigenvalue}.
The generalized eigenvalues and eigenvectors are also the solutions to the following \emph{generalized eigenvalue problem (equation form)}:
\begin{align}
\label{eq:generalized-eigenvalue-problem}
    L u = \lambda \tilde D u.
\end{align}
Thus, if $(\lambda, u)$ is a generalized eigenpair then $(\lambda, \tilde D^{1/2} u)$ is an eigenpair of $\tilde{\mathcal{L}}$.
A generalized eigenvector with a smaller generalized eigenvalue is smoother in terms of the variation $\Delta$.
Hence, the generalized eigenvalues are referred to as the \emph{frequency of the graph}.

The graph Fourier transform is a basis expansion by the generalized eigenvectors.
Let $U = [u_1, \dots, u_n]$ be the matrix whose columns are the generalized eigenvectors.
Then, the \emph{graph Fourier transform $\mathcal{F} \colon \mathbb{R}^n \to \mathbb{R}^n$} is defined by $\mathcal{F} x = \hat{x} := U^\top \tilde D x $, 
and the \emph{inverse graph Fourier transform $\mathcal{F}^{-1}$} is defined by $\mathcal{F}^{-1} \hat{x} = U \hat{x}$.
Note that these are actually the inverse transforms since $\mathcal{F} \mathcal{F}^{-1} = U^\top \tilde D U = I$.

The \emph{Parseval identity} relates the norms of the data and its Fourier transform:
\begin{align}
\label{eq:parseval}
    \| x \|_{\tilde D} = \| \hat{x} \|_2.
\end{align}

Let $h \colon \mathbb{R} \to \mathbb{R}$ be an analytic function.
The \emph{graph filter specified by $h$} is a mapping $x \mapsto y$ defined by the relation in the frequency domain: $\hat{y}(\lambda) = h(\lambda) \hat{x}(\lambda)$.
In the spatial domain, the above equation is equivalent to
$y = h(\L) x$.
where $h(\L)$ is defined via the Taylor expansion of $h$; see \cite{higham2008functions} for the detail of matrix functions.

In a machine learning problem on a graph, each vertex $i \in \mathcal{V}$ has a $d$-dimensional feature $\vec{x}(i) \in \mathbb{R}^d$.
We regard the features as $d$ graph signals and define the graph Fourier transform of the features by the graph Fourier transform of each signal.
Let $X = [\vec{x}(1); \dots, \vec{x}(n)]^\top$ be the feature matrix.
Then, the graph Fourier transform is represented by $\mathcal{F} X = \hat{X} =: U^\top \tilde D X$ and the inverse transform is $\mathcal{F}^{-1} \hat{X} = U \hat{X}$. We denote
$\hat{X} = [\hat{\vec{x}}(\lambda_1); \dots; \hat{\vec{x}}(\lambda_n)]^\top$ as the frequency components of $X$.

\section{Empirical Evidence of Assumption~\ref{asmp:low-frequency}} \label{sec:low-freq}

The results of this paper deeply depend on Assumption~\ref{asmp:low-frequency}. Thus, we first verify this assumption in real-world datasets: Cora, Citeseer, and Pubmed~\cite{coraciteseerpubmed}. These are citation networks, in which vertices are scientific papers and edges are citations. We consider the following experimental setting: \textbf{1}.~Compute the graph Fourier basis $U$ from $\tilde{\mathcal{L}}$; \textbf{2}.~Add Gaussian noise to the input features: $\mathcal{X} \leftarrow \mathcal{X} + \mathcal{N}(0, \sigma^2)$ for $\sigma = \{0, 0.01, 0.05\}$; \textbf{3}.~Compute the first $k$-frequency component: $\hat{\mathcal{X}}_k = U[:k]^\top\tilde{D}^{1/2}\mathcal{X}$; \textbf{4}.~Reconstruct the features: $\tilde{\mathcal{X}}_k = \tilde{D}^{-1/2}U[:k]\hat{\mathcal{X}}_k$; \textbf{5}.~Train and report test accuracy of a 2-layers perceptron on the reconstructed features $\tilde{\mathcal{X}}_k$.

\begin{figure}[h]
  \centering
  \includegraphics[width=\textwidth]{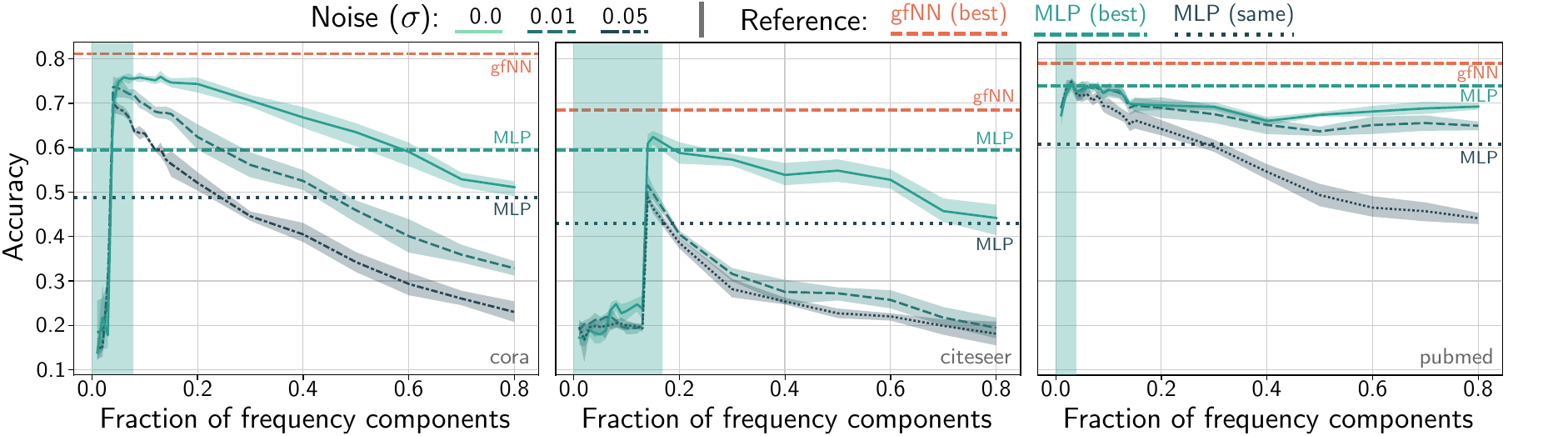}
  \caption{Average performance of 2-layers MLPs on frequency-limited feature vectors (\texttt{epochs=20}). The annotated ``(best)'' horizontal lines show the performance of the best performance using gfNN and a 2-layers MLP trained on the original features.}
  \label{f:ccp_freq}
\end{figure}

In Figure~\ref{f:ccp_freq}, we incrementally add normalized Laplacian frequency components to reconstruct feature vectors and train a 2-layers MLPs. We see that all three datasets exhibit a low-frequency nature. The classification accuracy of a 2-layers MLP tends to peak within the top 20\% of the spectrum (green boxes). By adding artificial Gaussian noise, we observe that the performance at low-frequency regions is relatively robust, which implies a strong denoising effect.

Interestingly, the performance gap between graph neural network and a simple MLP is much larger when high-frequency components do not contain useful information for classification. In the Cora dataset, high-frequency components only decrease classification accuracy. Therefore, our gfNN outperformed a simple MLP. Citeseer and Pubmed have a similar low-frequency nature. However, the performance gaps here are not as large. Since the noise-added performance lines for Citeseer and Pubmed generally behave like the original Cora performance line, we expect the original Citeseer and Pubmed contain much less noise compared with Cora. Therefore, we could expect the graph filter degree would have little effect in Citeseer and Pubmed cases.  

\section{Multiplying Adjacency Matrix is Low Pass Filtering}
\label{sec:lowpass}


Computing the low-frequency components is expensive since it requires $O(|\mathcal{V}|^3)$ time to compute the eigenvalue decomposition of the Laplacian matrix.
Thus, a reasonable alternative is to use a low-pass filter.
Many papers on graph neural networks iteratively multiply the (augmented) adjacency matrix $\A$ (or $\tilde A$) to propagate information.
In this section, we see that this operation corresponds to a \emph{low-pass filter}.

Multiplying the normalized adjacency matrix corresponds to applying graph filter $h(\lambda) = 1 - \lambda$.
Since the eigenvalues of the normalized Laplacian lie on the interval $[0, 2]$, this operation resembles a \emph{band-stop filter} that removes intermediate frequency components.
However, since the maximum eigenvalue of the normalized Laplacian is $2$ if and only if the graph contains a non-trivial bipartite graph as a connected component~\cite[Lemma~1.7]{chung1997spectral}.
Therefore, for other graphs, multiplying the normalized (non-augmented) adjacency matrix acts as a low-pass filter (i.e., high-frequency components must decrease).

We can increase the low-pass filtering effect by adding self-loops (i.e., considering the augmented adjacency matrix) since it shrinks the eigenvalues toward zero as follows.%
\footnote{The shrinking of the maximum eigenvalue has already been proved in \cite[Theorem~1]{wu2019simplifying}. Our theorem is stronger than theirs since we show the ``monotone shrinking'' of ``all'' the eigenvalues. In addition, our proof is simpler and shorter.}
\begin{theorem}
\label{thm:spectrum-shrinking}
Let $\lambda_i(\gamma)$ be the $i$-th smallest generalized eigenvalue of $(\tilde D, L) = (D + \gamma I)$.
Then, $\lambda_i(\gamma)$ is a non-negative number, and monotonically non-increasing in $\gamma \ge 0$.
Moreover, $\lambda_i(\gamma)$ is strictly monotonically decreasing if $\lambda_i(0) \neq 0$.
\end{theorem}
Note that $\gamma$ cannot be too large. 
Otherwise, all the eigenvalues would be concentrated around zero, i.e., all the data would be regarded as ``low-frequency''; hence, we cannot extract useful information from the low-frequency components. 

The graph filter $h(\lambda) = 1 - \lambda$ can also be derived from a first-order approximation of a Laplacian regularized least squares~\cite{belkin2004semi}.
Let us consider the problem of estimating a low-frequency true feature $\{ \bar{\vec{x}}(i) \}_{i \in \mathcal{V}}$ from the observation $\{ \vec{x}(i) \}_{i \in \mathcal{V}}$.
Then, a natural optimization problem is given by
\begin{align}
    \min_{\bar{x}} \sum_{i \in \mathcal{V}} \| \bar{\vec{x}}(i) - \vec{x}(i) \|_{\tilde D}^2 + \Delta(\bar{X})
\end{align}
where $\Delta(\bar{X}) = \sum_{p = 1}^d \Delta(\bar{\vec{x}}(i)_p)$.
Note that, since $\Delta(\bar{X}) = \sum_{\lambda \in \Lambda} \lambda \| \hat{\bar{\vec{x}}}(\lambda) \|_2^2$, it is a maximum a posteriori estimation with the prior distribution of $\hat{\bar{\vec{x}}}(\lambda) \sim N(0, I/\lambda)$.
The optimal solution to this problem is given by $\bar{X} = (I + \L)^{-1} X$.
The corresponding filter is $h'(\lambda) = (1 + \lambda)^{-1}$, and hence $h(\lambda) = 1 - \lambda$ is its first-order Taylor approximation.

\section{Bias-Variance Trade-off for Low Pass Filters}

In the rest of this paper, we establish theoretical results under Assumption~\ref{asmp:low-frequency}.
To be concrete, we pose a more precise assumption as follows.

\begin{assumption}[Precise Version of the First Part of Assumption~\ref{asmp:low-frequency}]
\label{asmp:low-frequency-precise}
Observed features $\{ \vec{x}(i) \}_{i \in \mathcal{V}}$ consists of true features $\{ \vec{\bar{x}}(i) \}_{i \in \mathcal{V}}$ and noise  $\{ \vec{z}(i) \}_{i \in \mathcal{V}}$.
The true features $\bar{X}$ have frequency at most $0 \le \epsilon \ll 1$ and the noise follows a \emph{white Gaussian noise}, i.e., each entry of the graph Fourier transform of $Z$ independently identically follows a normal distribution $N(0, \sigma^2)$.
\end{assumption}

Using this assumption, we can evaluate the effect of the low-pass filter as follows.
\begin{lemma}
\label{lem:bias-variance}
Suppose Assumption~\ref{asmp:low-frequency-precise}.
For any $0 < \delta < 1/2$, with probability at least $1 - \delta$, we have
\begin{align}
\label{eq:bias-variance}
    \| \bar{X} - \A^k X \|_D
    \le \sqrt{k \epsilon} \| \bar{X} \|_D
    + O\left(\sqrt{\log (1 / \delta) R(2 k)}\right) \E[ \| Z \|_D ],
\end{align}
where $R(2 k)$ is a probability that a random walk with a random initial vertex returns to the initial vertex after $2 k$ steps.
\end{lemma}
The first and second terms of the right-hand side of \eqref{eq:bias-variance} are the \emph{bias term} incurred by applying the filter and the \emph{variance term} incurred from the filtered noise, respectively.
Under the assumption, the bias increases a little, say, $O(\sqrt{\epsilon})$.
The variance term decreases in $O(1/\mathrm{deg}^{k/2})$ where $\mathrm{deg}$ is a typical degree of the graph since $R(2 k)$ typically behaves like $O(1 / \mathrm{deg}^k)$ for small $k$.%
\footnote{This exactly holds for a class of locally tree-like graphs~\cite{dembo2013factor}, which includes Erdos--Renyi random graphs and the configuration models.}
Therefore, we can obtain a more accurate estimation of the true data from the noisy observation by multiplying the adjacency matrix if the maximum frequency of $\bar{X}$ is much smaller than the \emph{noise-to-signal ratio} $\| Z \|_D / \| \bar{X} \|_D$.

This theorem also suggest a choice of $k$. By minimizing the right-hand side by $k$, we obtain:
\begin{corollary}
\label{cor:optimal-k}
Suppose that $\E[ \| Z \|_D ] \le \rho \| \bar{X} \|_D$ for some $\rho = O(1)$.
Let $k^*$ be defined by $k^* = O(\log (\log (1 / \delta) \rho / \epsilon ))$, and suppose that there exist constants $C_d$ and $\bar{d} > 1$ such that $R(2 k) \le C_d/\bar{d}^k$ for $k \le k^*$.
Then, by choosing $k = k^*$, the right-hand side of \eqref{eq:bias-variance} is $\tilde{O}(\sqrt{\epsilon})$.\footnote{$\tilde{O}(\cdot)$ suppresses logarithmic dependencies.}
\qed
\end{corollary}
This concludes that we can obtain an estimation of the true features with accuracy $\tilde{O}(\sqrt{\epsilon})$ by multiplying $\A$ several times.

\section{Graph Filter Neural Network}

In the previous section, we see that low-pass filtered features $\A^k X$ are accurate estimations of the true features with high probability. In this section, we analyze the performance of this operation.

Let the multi-layer (two-layer) perceptron be
\begin{align}
    h_\text{MLP}(X | W_1, W_2) &= \sigma_2(\sigma_1(X W_1) W_2),
\end{align}
where $\sigma_1$ is the entry-wise ReLU function, and $\sigma_2$ is the softmax function.
Note that both $\sigma_1$ and $\sigma_2$ are contraction maps, i.e., $\|\sigma_i(X) - \sigma(Y)\|_D \le \|X - Y\|_D$.

Under the second part of Assumption~\ref{asmp:low-frequency}, our goal is to obtain a result similar to $h(\bar{X} | W_1, W_2)$.
The simplest approach is to train a multi-layer perceptron $h_\text{MLP}(X | W_1, W_2)$ with the observed feature.
The performance of this approach is evaluated by
\begin{align}
\label{eq:naive-approach}
    \| h_\text{MLP}(\bar{X} | W_1, W_2) - h_\text{MLP}(X | W_1, W_2) \|_{\tilde D} \le \| Z \|_D \rho(W_1) \rho(W_2),
\end{align}
where $\rho(W)$ is the maximum singular value of $W$.

Now, we consider applying graph a filter $\A^k$ to the features, then learning with a multi-layer perceptron, i.e., $h_\text{MLP}(h(\L) X | W_1, W_2)$.
Using Lemma~\ref{lem:bias-variance}, we obtain the following result.
\begin{theorem}
\label{thm:MLP-vs-gfNN}
Suppose Assumption~\ref{asmp:low-frequency-precise} and conditions in Corollary~\ref{cor:optimal-k}. 
By choosing $k$ as Corollary~\ref{cor:optimal-k}, with probability at least $1 - \delta$ for $\delta < 1/2$, 
\begin{align}
    \| h_\text{MLP}(\bar{X} | W_1, W_2) - h_\text{MLP}(\A^2 X | W_1, W_2) \|_D = \tilde{O}(\sqrt{\epsilon}) \E[ \| Z \|_D ] \rho(W_1) \rho(W_2).
\end{align}
\qed
\end{theorem}
This means that if the maximum frequency of the true data is small, we obtain a solution similar to the optimal one.

\paragraph{Comparison with GCN}

Under the same assumption, we can analyze the performance of a two-layers GCN.
The GCN is given by
\begin{align}
    h_\text{GCN}(X | W_1, W_2) &= \sigma_2(\A \sigma_1(\A X W_1) W_2).
\end{align}
\begin{theorem}
\label{thm:MLP-vs-GCN}
Under the same assumption to Theorem~\ref{thm:MLP-vs-gfNN}, we have
\begin{align}
    \| h_\text{MLP}(\bar{X} | W_1, W_2) - h_\text{GCN}(X | W_1, W_2) \|_D \le \tilde{O}(\sqrt[4]{\epsilon}) \E[ \| Z \|_D ] \rho(W_1) \rho(W_2).
\end{align}
\end{theorem}
This theorem means that GCN also has a performance similar to MLP for true features.
Hence, in the limit of $\epsilon \to 0$ and $Z \to 0$, all MLP, GCN, and the gfNN have the same performance.

In practice, gfNN has two advantages over GCN. First, gfNN is faster than GCN since it does not use the graph during the training. Second, GCN has a risk of overfitting to the noise. When the noise is so large that it cannot be reduced by the first-order low-pass filter, $\A$, the inner weight $W_1$ is trained using noisy features, which would overfit to the noise. We verify this in Section~\ref{exp:addgaussian}. 

\paragraph{Comparison with SGC}

Recall that a degree-2 SGC is given by
\begin{align}
    h_\text{SGC}(X | W_1) = \sigma_2(\A^2 X W_2),
\end{align}
i.e., it is a gfNN with a one-layer perceptron. Thus, by the same discussion, SGC has a performance similar to the perceptron (instead of the MLP) applied to the true feature. This clarifies an issue of SGC: if the true features are non-separable, SGC cannot solve the problem.
We verify this issue in experiment~\textbf{E2} (Section~\ref{exp:donuts}).  

\section{Experiments}

To verify our claims in previous sections, we design two experiments. In experiment \textbf{E1}, we add different levels of white noise to the feature vectors of real-world datasets and report the classification accuracy among baseline models. Experiment \textbf{E2} studies an artificial dataset which has a complex feature space to demonstrate when simple models like SGC fail to classify.
  
There are two types of datasets: Real-worlds data (citation networks, social networks, and biological networks) commonly used for graph neural network benchmarking~\cite{coraciteseerpubmed, planetoid, zitnik2017predicting} and artificially synthesized random graphs from the two circles dataset. We created the graph structure for the two circles dataset by making each data point into a vertex and connect the 5 closest vertices in Euclidean distance. Table~\ref{t:data} gives an overview of each dataset.

\begin{table}[h]
  \vspace{-0.5em}
  \caption{Real-world benchmark datasets and synthetic datasets for vertex classification} \label{t:data}
  \centering
  \begin{tabular}{llllll}
    \toprule
    \bf{Dataset}   & \bf{Nodes} & \bf{Edges} & \bf{Features} & \bf{Classes}  & \bf{Train/Val/Test Nodes} \\
    \midrule
    \bf{Cora}        & 2,708      & 5,278  & 1,433         & 7             & 140/500/1,000             \\
    \bf{Citeseer}   & 3,327      & 4,732      & 3,703         & 6             & 120/500/1,000             \\
    \bf{Pubmed}     & 19,717     & 44,338     & 500           & 3             & 60/500/1,000              \\
    \bf{Reddit}        & 231,443    & 11,606,919 & 602           & 41            & 151,708/23,699/55,334     \\
    \bf{PPI}        & 56,944    & 818,716 & 50           & 121            & 44,906/6,514/5,524     \\
    \midrule
    \bf{Two Circles}    & 4,000      & 10,000     & 2                    & 2 & 80/80/3,840 \\
    \bottomrule
  \end{tabular}
  \vspace{-0.5em}
\end{table}

\subsection{Neural Networks}

We compare our results with a few state-of-the-art graph neural networks on benchmark datasets. We train each model using Adam optimizer~\cite{kingma2014adam} (\texttt{lr=0.2, epochs=50}). We use 32 hidden units for the hidden layer of GCN, MLP, and gfNN. Other hyperparameters are set similar to SGC~\cite{wu2019simplifying}.

\textbf{gfNN} We have three graph filters for our simple model: \emph{Left Norm} 
(\textcolor{ror}{$\blacktriangledown$}), \emph{Augumented Normalized Adjacency} 
(\textcolor{bor}{$\blacktriangle$}), and \emph{Bilateral} (\textcolor{yel}{$\blacklozenge$}).

\textbf{SGC \cite{wu2019simplifying}} Simple Graph Convolution (\textcolor{gre}{\small $\blacksquare$}) simplifies the Graph Convolutional Neural Network model by removing nonlinearity in the neural network and only averaging features.

\textbf{GCN \cite{gcn}} Graph Convolutional Neural Network (\textcolor{blu}{\tiny \XSolidBold{}}) is the most commonly used baseline.

\begin{figure}[h]
  \centering
  \includegraphics[width=\textwidth]{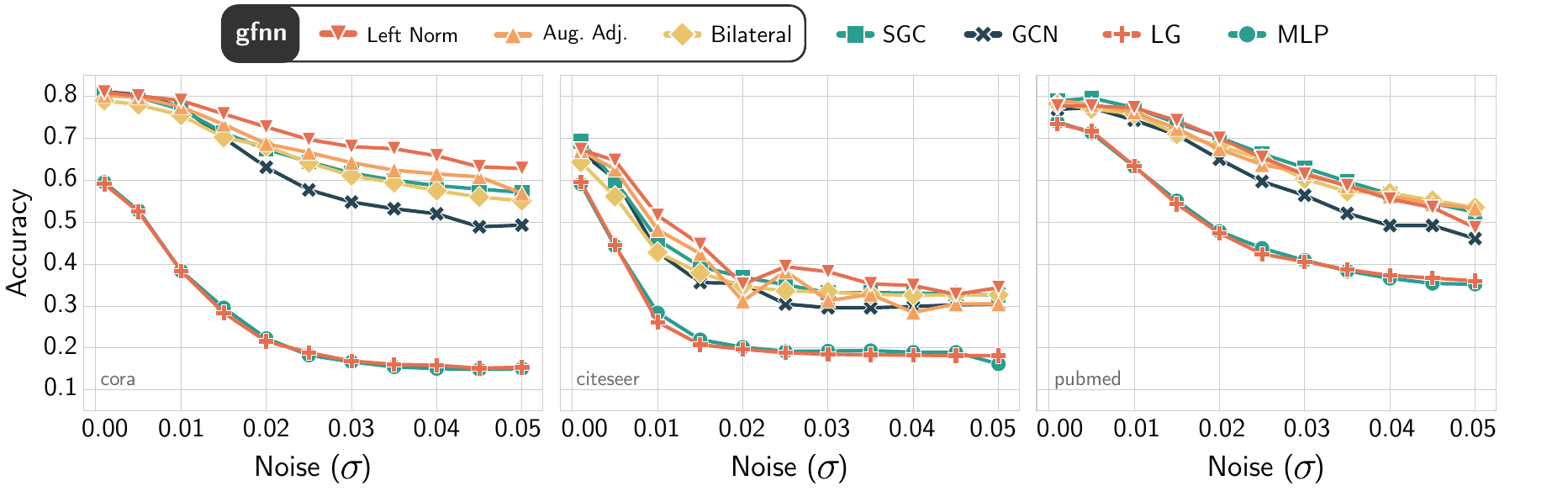}
  \caption{Benchmark test accuracy on Cora (left), Citeseer (middle), and Pubmed (right) datasets. The noise level is measured by standard deviation of white noise added to the features.}
  \label{f:ccp_noise}
\end{figure}

\subsection{Denoising Effect of Graph Filters} \label{exp:addgaussian}

For each dataset in Table \ref{t:data}, we introduce a white noise $\mathcal{N}(0, \delta^2)$ into the feature vectors with $\delta$ in range $(0.01, 0.05)$. According to the implications of Theorem~\ref{thm:MLP-vs-GCN} and Theorem~\ref{thm:MLP-vs-gfNN}, GCN should exhibit a low tolerance to feature noise due to its first-order denoising nature. As the noise level increases, we see in Figure~\ref{f:ccp_noise} that GCN, Logistic Regression (LR), and MLP tend to overfit more to noise. On the other hand, gfNN and SGC remain comparably noise tolerant.

\subsection{Expressiveness of Graph Filters} \label{exp:donuts}

Since graph convolution is simply denoising, we expect it to have little contribution to non-linearity learning. Therefore, we implement a simple 2-D two circles dataset to see whether graph filtering can have the non-linear manifold learning effect. Based on the Euclidean distance between data points, we construct a $k$-nn graph to use as the underlying graph structure. In this particular experiment, we have each data point connected to the 5 closest neighbors. Figure~\ref{f:donuts} demonstrates how SGC is unable to classify a simple 500-data-points set visually.

\begin{figure}[h]
  \centering
  \includegraphics[width=0.9\textwidth]{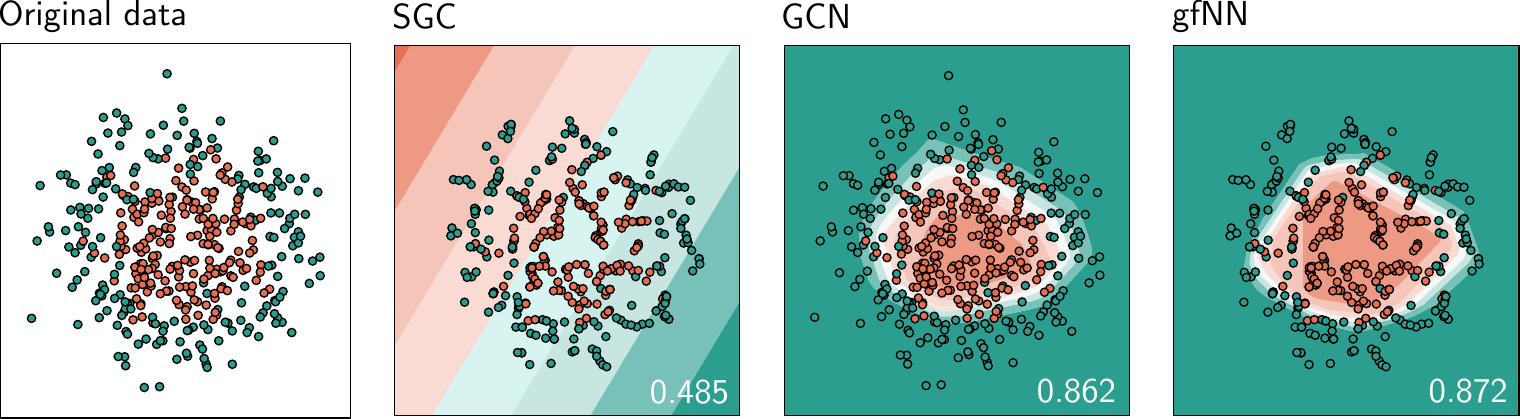}
  \caption{Decision boundaries on 500 generated data samples following the two circles pattern}
  \label{f:donuts}
\end{figure}

On the other hand, SGC, GCN, and our model gfNN perform comparably on most benchmark graph datasets. Table~\ref{t:modelcomp} compares the accuracy (f1-micro for Reddit) of randomized train/test sets. We see that all these graph neural networks have similar performance (some minor improvement might subject to hyperparameters tuning). Therefore, we believe current benchmark datasets for graph neural network have quite ``simple'' feature spaces. 

In practice, both of our analysis and experimental results suggest that the graph convolutional layer should be thought of simply as a denoising filter. In contrast to the recent design trend involving GCN~\cite{abu2019mixhop,valsesia2018learning} , our results imply that simply stacking GCN layers might only make the model more prone to feature noise while having the same expressiveness as a simple MLP.

\begin{table}[h]
  \vspace{-0.5em}
  \caption{Average test accuracy on random train/val/test splits (5 times)}\label{t:modelcomp}
  \centering
  \resizebox{\textwidth}{!}{%
  \begin{tabular}{l|ll|llllll}
    \toprule
    \bf{}              & \bf{Denoise} & \bf{Classifier} & \bf{Cora}& \bf{Citeseer}&\bf{Pubmed}&\bf{Reddit}&\bf{PPI}& \bf{Two Circles} \\
    \midrule
    \bf{GCN}           & 1st order   & \textit{Non-linear} &$80.0\pm1.8$&${69.6}\pm1.1$&$79.3\pm1.3$&OOM &OOM& $83.9\pm3.1$  \\
    \bf{SGC}           & \textit{2nd order}   & Linear  &$77.6\pm2.2$ &$65.6\pm0.1$&$78.4\pm1.1$&${94.9}\pm1.2$& $62.0\pm1.2$  & $54.5\pm4.4$   \\
    \midrule
    \bf{gfNN}   & \textit{2nd order}   & \textit{Non-linear} &${80.9}\pm1.3$&$ 69.3\pm1.1$&${81.2}\pm1.5$&$92.4\pm2.2$& ${62.2}\pm1.5$ &${84.6}\pm2.5$  \\
    \bottomrule
  \end{tabular}%
  }
  \vspace{-0.5em}
\end{table}

\section{Discussion}

There has been little work addressed the limits of the GCN architecture. \citet{kawamoto2018mean} took the mean-field approach to analyze a simple GCN model to statistical physics. They concluded that backpropagation improves neither accuracy nor detectability of a GCN-based GNN model. \citet{li2018deeper} empirically analyzed GCN models with many layers under the limited labeled data setting and stated that GCN would not do well with little labeled data or too many stacked layers. While these results have provided an insightful view for GCN, they have not insufficiently answered the question: \emph{When we should we use GNN?}.

Our results indicate that we \emph{should} use the GNN approach to solve a given problem if Assumption~\ref{asmp:low-frequency} holds. From our perspective, GNNs derived from GCN simply perform noise filtering and learn from denoised data. Based on our analysis, we propose two cases where GCN and SGC might fail to perform: noisy features and nonlinear feature spaces. In turn, we propose a simple approach that works well in both cases.

Recently GCN-based GNNs have been applied in many important applications such as point-cloud analysis \cite{valsesia2018learning} or weakly-supervised learning \cite{garcia2017few}. As the input feature space becomes complex, we advocate revisiting the current GCN-based GNNs design. Instead of viewing a GCN layer as the convolutional layer in computer vision, we need to view it simply as a denoising mechanism. Hence, simply stacking GCN layers only introduces overfitting and complexity to the neural network design.

\pagebreak

\bibliographystyle{plainnat}
\bibliography{neurips_2019}

\clearpage

\appendix

\input{appendix.tex}

\end{document}

%% file: appendix.tex
\section{Proofs}

\begin{proof}[Proof of Theorem~\ref{thm:spectrum-shrinking}]
Since the generalized eigenvalues of $(D + \gamma I, L)$ are the eigenvalues of a positive semidefinite matrix $(D + \gamma I)^{1/2} L (D + \gamma I)^{1/2}$, these are non-negative real numbers.
To obtain the shrinking result, we use the Courant--Fisher--Weyl's min-max principle~\cite[Corollary III. 1.2]{bhatia2013matrix}:
For any $0 \le \gamma_1 < \gamma_2$, 
\begin{align}
    \lambda_i(\gamma_2) 
    &= \min_{H: \text{subspace}, \mathrm{dim}(H) = i} \max_{x \in H, x \neq 0} \frac{x^\top L x}{x^\top (D + \gamma_2 I) x} \\
    &\le \min_{H: \text{subspace}, \mathrm{dim}(H) = i} \max_{x \in H, x \neq 0} \frac{x^\top L x}{x^\top (D + \gamma_1 I) x} \\
    &= \lambda_i(\gamma_1).
\end{align}
Here, the second inequality follows because $x^\top (D + \gamma_1) x < x^\top (D + \gamma_2) x$ for all $x \neq 0$
Hence, the inequality is strict if $x^\top L x \neq 0$, i.e., $\lambda_i(\gamma_1) \neq 0$.
\end{proof}

\begin{lemma}
\label{lem:laplacian-times-low-frequecy}
If $X$ has a frequency at most $\epsilon$ then $ \| h(\mathcal{L}) X \|_D^2 \le \max_{t \in [0, \epsilon]} h(t) \| X \|_D^2$.
\end{lemma}
\begin{proof}
By the Parseval identity,
\begin{align}
    \| \mathcal{L} X \|_D^2 
    &= \sum_{\lambda \in \Lambda} h(\lambda) \| \hat{\vec{x}}(\lambda) \|_2^2 \\
    &\le \max_{t \in [0,\epsilon]} h(t) \sum_{\lambda \in \Lambda} \| \hat{\vec{x}}(\lambda) \|_2^2 \\
    &= \max_{t \in [0,\epsilon]} h(t) \| \hat{X} \|_2^2  \\
    &= \max_{t \in [0,\epsilon]} h(t) \| \hat{X} \|_D^2.
\end{align}
\end{proof}

\begin{proof}[Proof of Lemma~\ref{lem:bias-variance}]
By substituting $X = \bar{X} + Z$, we obtain
\begin{align}
    \| \bar{X} - \A^k X \|_{\tilde D}
    &\le
    \| (I - \A^k) \bar{X} \|_{\tilde D}
    +
    \| \A^k Z \|_{\tilde D}.
\end{align}
By Lemma~\ref{lem:laplacian-times-low-frequecy}, the first term is bounded by $\sqrt{k \epsilon} \| \bar{X} \|_D$.
By the Parseval identity~\eqref{eq:parseval}, the second term is evaluated by
\begin{align}
    \| \A^k Z \|_{\tilde D}^2 
    &= \sum_{\lambda \in \Lambda} (1 - \lambda)^{2 k} \| \hat{\vec{z}}(\lambda) \|_2^2 \\
    &= \sum_{\lambda \in \Lambda, p \in \{1, \dots, d\}} (1 - \lambda)^{2 k} \hat{z}(\lambda, p)^2
\end{align}
By \cite[Lemma~1]{laurent2000adaptive}, we have
\begin{align}
    \mathrm{Prob}\left\{ \sum_{\lambda,p} (1 - \lambda)^{2 k} (\hat{z}(\lambda, p)^2/\sigma^2 - 1) \ge 2 \sqrt{ t \sum_{\lambda,p} (1 - \lambda)^{4 k}} + 2 t \right\} \le e^{-t}.
\end{align}
By substituting $t = \log (1 / \delta)$ with $\delta \le 1/2$, we obtain
\begin{align}
    \mathrm{Prob}\left\{ \sum_{\lambda,p} (1 - \lambda)^{2 k} \hat{z}(\lambda, p)^2 \ge 5 \sigma^2 n d \log (1/\delta) \frac{\sum_{\lambda} (1 - \lambda)^{2 k}}{n} \right\} \le 1 - \delta.
\end{align}
Note that $\E[ \| Z \|_D^2 = \sigma^2 n d $, and 
\begin{align}
    \frac{\sum_{\lambda} (1 - \lambda)^{2 k}}{n} = \frac{\mathrm{tr}(\A^{2k})}{n} = R(2 k)
\end{align}
since $(i, j)$ entry of $\A^{2 k}$ is the probability that a random walk starting from $i \in \mathcal{V}$ is on $j \in \mathcal{V}$ at $2 k$ step, we obtain the lemma.
\end{proof}

\begin{proof}[Proof of Theorem~\ref{thm:MLP-vs-gfNN}]
By the Lipschitz continuity of $\sigma$, we have
\begin{align}
    \| h_\text{MLP}(\bar{X} | W_1, W_2) - h_\text{MLP}(\A^k X | W_1, W_2) \|_D 
    &\le \| \bar{X} - \A^k X \|_D \rho(W_1) \rho(W_2).
\end{align}
By using Lemma~\ref{lem:bias-variance}, we obtain the result.
\end{proof}

\begin{lemma}
\label{lem:laplacian-times-activated-low-frequecy}
If $X$ has a frequency at most $\epsilon$ then there exists $Y$ such that $Y$ has a frequency at most $\sqrt{\epsilon}$ and $\| \sigma(X) - Y \|_D^2 \le \sqrt{\epsilon} \| X \|_D^2$.
\end{lemma}
\begin{proof}
We choose $Y$ by truncating the frequency components greater than $\sqrt{\epsilon}$. 
Then,
\begin{align}
    \| \sigma(X) - Y \|_D^2 
    &= \sum_{\lambda > \sqrt{\epsilon}} \| \widehat{\sigma(X)}(\lambda) \|^2 \\
    &\le \frac{1}{\sqrt{\epsilon}} \sum_{\lambda} \lambda \| \widehat{\sigma(X)}(\lambda) \|^2 \\
    &= \frac{1}{\sqrt{\epsilon}} \sum_{(u, v) \in E} \| \sigma(x(u)) - \sigma(x(v)) \|_2^2 \\
    &\le \frac{1}{\sqrt{\epsilon}} \sum_{(u, v) \in E} \| x(u) - x(v) \|_2^2 \\
    &= \sqrt{\epsilon} \sum_{\lambda} \lambda \| \hat{x}(\lambda) \|_2^2 \\
    &\le \sqrt{\epsilon} \sum_{\lambda} \| \hat{x}(\lambda) \|_2^2 \\
    &= \sqrt{\epsilon} \| X \|_D^2.
\end{align}
\end{proof}

\begin{proof}[Proof of Theorem~\ref{thm:MLP-vs-GCN}]
\begin{align}
     &\| \sigma(\bar{X} W_1) W_2 - \A \sigma(\A X W_1) W_2 \|_D \\
     &\le 
     \left( \| \sigma(\bar{X} W_1) - \A \sigma(\bar{X} W_1) \|_D
     +
     \| \A \sigma(\bar{X} W_1) - \A \sigma(\A X W_1) \|_D \right) \rho(W_2) \\
     &= \left( \| \L \sigma(\bar{X} W_1) \|_D + \| \bar{X} - \A X \|_D \rho(W_1) \right) \rho(W_2) \\
     &\le \left( O(\epsilon^{1/4}) \| \bar{X} \|_D + R(2) \| Z \|_D \right) \rho(W_1) \rho(W_2).
\end{align}
\end{proof}